\documentclass[journal]{IEEEtran}
\IEEEoverridecommandlockouts                              
\bstctlcite{IEEEexample:BSTcontrol} 

\usepackage[utf8]{inputenc} 
\usepackage[T1]{fontenc}    
\usepackage{hyperref}       
\usepackage{url}            
\usepackage{booktabs}       
\usepackage{amsfonts}       
\usepackage{nicefrac}       
\usepackage{microtype}      
\usepackage[square,sort,comma,numbers]{natbib}

\usepackage{comment}
\usepackage{lmodern}
\usepackage{scalefnt}
\usepackage{setspace}
\usepackage{romannum}
\usepackage[usenames,dvipsnames]{xcolor}
\usepackage{tkz-berge}
\usetikzlibrary{fit,shapes}
\usepackage{calrsfs}
\DeclareMathAlphabet{\pazocal}{OMS}{zplm}{m}{n}
\usepackage{graphicx}
\usepackage{graphics}
\usepackage{epsfig}
\usepackage{mathptmx}
\usepackage{caption}
\usepackage{subcaption}
\usepackage{tikz}
\usetikzlibrary{positioning}
\usepackage{amsmath, amssymb}
\usetikzlibrary{arrows}
\usepackage{dsfont}

\usepackage{graphicx}
\usepackage{subcaption}

\usepackage{kantlipsum}
\usepackage{amsthm}
\usepackage{amsfonts}
\usepackage{setspace}
\usepackage{multirow}
\usepackage[english]{babel}
\usepackage{float}
\usepackage{algorithmicx}
\usepackage{algorithm}
\usepackage{algpascal}
\usepackage{algc}
\usepackage{algcompatible}
\usepackage{multicol}
\usepackage{algpseudocode}
\let\oldReturn\Return
\renewcommand{\Return}{\State\oldReturn}
\usepackage{mathtools}
\usepackage{mathptmx}
\usepackage{mathtools, cuted}
\usepackage{textcomp}
\usepackage[english]{babel}
\usepackage{rotating}
\usepackage{pgfplots}
\pgfplotsset{compat=1.5}
\usepackage{siunitx}
\usepackage{pgfplotstable}
\usepackage{filecontents}

\usepackage{bbold}
\usepackage{float}

\newtheorem{theorem}{Theorem}

\newtheorem{lemma}[theorem]{Lemma}
\newtheorem{prop}[theorem]{Proposition}

\newtheorem{assume}{Assumption}

\DeclareMathAlphabet{\pazocal}{OMS}{zplm}{m}{n}

\newcommand{\ols}[1]{\mskip.5\thinmuskip\overline{\mskip-.5\thinmuskip {#1} \mskip-.5\thinmuskip}\mskip.5\thinmuskip} 
\renewcommand{\bar}[1]{\mskip.5\thinmuskip\overline{\mskip-.5\thinmuskip {#1} \mskip-.5\thinmuskip}\mskip.5\thinmuskip} 
\newcommand{\norm}[1]{\left\lVert#1\right\rVert}
\newcommand{\I}{\pazocal{I}}
\newcommand{\X}{\pazocal{X}}
\newcommand{\F}{\pazocal{F}}

\newcommand{\R}{\pazocal{R}}

\newcommand{\xs}{x^{\star}}
\newcommand{\A}{{\pazocal{A}}}
\newcommand{\C}{{\pazocal{C}}}
\newcommand{\D}{{\pazocal{D}}}
\newcommand{\B}{{\pazocal{B}}}
\newcommand{\E}{{\pazocal{E}}}
\renewcommand{\S}{{\pazocal{S}}}

\renewcommand{\P}{{\pazocal{ P}}}

\newcommand{\pxc}{\phi_{x_t, c_t}}

\newcommand{\psit}{{ \psi}_{x_\ti, \mu_t}}

\renewcommand{\det}{{\mathrm{det}}}
\newcommand{\good}{g}
\newcommand{\bad}{b}
\newcommand{\xti}{x_{t, i}}
\newcommand{\yti}{y_{t, i }}
\newcommand{\ti}{{t, i }}
\newcommand{\thetas}{\theta^{\star}}
\newcommand{\isyn}{{\mathrm{syn}}}
\newcommand{\inew}{{t, i}}
\newcommand{\trace}{{\mathrm{trace}}}
\newcommand{\Rb}{{\mathbb{R}}}

\renewcommand{\l}{{j-1}}
\newcommand{\Rgood}{\R_{\good}(T)}
\newcommand{\Rbad}{\R_{\bad}(T)}
\newcommand{\hii}{i}
\newcommand{\htt}{t}

\begin{document}
\title{\LARGE \bf Distributed Stochastic Bandit Learning with  Delayed Context Observation}
\author{Jiabin~Lin and Shana Moothedath,\IEEEmembership{~Member, IEEE}
 \thanks{J. Lin and  S. Moothedath  are with the Department of Electrical and Computer Engineering, Iowa State University, Ames, IA 50011, USA. \texttt{\{jiabin, mshana\}@iastate.edu}.}
}
\maketitle
\thispagestyle{empty}
\pagestyle{empty}
\begin{abstract}
We consider the problem where $M$ agents collaboratively interact with an instance of a stochastic $K$-armed contextual bandit, where $K \gg M$. The goal of the agents is to simultaneously minimize the cumulative regret over all the agents over a time horizon $T$.  We consider a setting where the exact context is observed after a delay and at the time of choosing the action the agents are unaware of the  context and only a distribution on the set of contexts is  available. Such a situation arises in different applications where at the time of the decision  the context needs to be predicted (e.g., weather forecasting or stock market prediction), and the context can be estimated once the reward is obtained.  We propose an Upper Confidence Bound (UCB)-based distributed algorithm and prove that our algorithm achieves regret and communications bounds of $O(d\sqrt{MT}\log^2 T)$  and $ O(M^{1.5} d^{3}) $, respectively, for linearly parametrized reward functions. 
We validated the performance of our algorithm via numerical simulations on synthetic data and real-world Movielens  data.
\end{abstract}

\IEEEpeerreviewmaketitle

\section{Introduction}
Sequential decision making is a common problem in many applications, including control and robotics \cite{cheung2013autonomous, srivastava2014surveillance}, communications \cite{anandkumar2011distributed}, and ecology \cite{srivastava2013optimal}. 
Bandit algorithms provide a learning framework to model the sequential decision making problem where the learner interacts with the environment in several rounds and the goal of the learner is to choose the best action in each round to maximize the cumulative reward over a period of time   \cite{bubeck2012regret}. A popular variant of bandit algorithms is contextual bandits. In the standard contextual bandit model,  the learner  observes a context/feature vector, chooses an action and  receives a reward based on the context and the chosen action.  One of the key challenge in bandits is to balance the trade-off between exploring new actions in the pursuit of finding the best action and exploiting the known actions \cite{bubeck2012regret, lattimore2020bandit}.

Recently,  many papers studied MAB problems with multiple agents, where a set of agents/learners face the same MAB problem.  Collaboration among multiple agents  expedites the learning process in many applications that use contextual bandit algorithms, such as recommender systems, clinical trials, control and robotics, and cognitive radio  \cite{huang2021federated, wang2019distributed, landgren2021distributed}. However, often  the contexts are noisy or represent predictive measures, e.g., weather prediction or stock market prediction. In such scenarios, the exact contexts are not available  and learners  only observe a distribution on the set of contexts. There are many applications where the actions/decisions are made based on a prediction/distribution and the  exact contexts are  observed  after choosing an action   (e.g., we decide whether to take an umbrella based on the weather forecast and we know if it  rained later in the day). In such situations, the exact context is available to the learner after a delay and we refer to this MAB problem as {\em contextual bandits with delay}.

Our goal in this paper is to propose a communication cost-effective algorithm for distributed bandit learning with $M$ agents and delayed context observation.  Bandit learning with delayed  contexts is more challenging due to the fact that the learner do not have access to the  context information while choosing the action. In order to address this difficulty, we convert the problem using a feature mapping that is used in  \cite{kirschner2019stochastic}for a  single agent bandit problem. After modifying the problem, we add a new set of feature vectors such that the reward under this set of $d$-dimensional context feature vectors is an unbiased observation for the action selection. We propose a UCB-based {\em distributed} bandit algorithm with regret bound $ O(d \sqrt{M T} \log ^{2}(T)) $ for linearly parametrized reward functions; the order of our regret bound coincides with the regret bound of the distributed bandit algorithm in \cite{wang2019distributed} (\cite{wang2019distributed} assumed the exact contexts are known). Our setting recovers the distributed bandit algorithm with known contexts in \cite{wang2019distributed} when the context distribution is a Dirac delta distribution.

We note that there is a straightforward communication protocol for distributed bandit learning is {\em immediate sharing}  where each agent shares every new sample immediately with the other agents as noted in  \cite{wang2019distributed}. While the agents can achieve near-optimal regret under this protocol, the amount of communication data is directly proportional to the total size of gathered samples, rendering the problem non-scalable over large time horizons. To minimize the communication cost while retaining optimum regret, we use the observation in \cite{abbasi2011improved} and execute synchronization between agents only when the extra information accessible to the agents is significant when compared to the last synchronization. 

This paper makes the following contributions.  

\begin{enumerate}
\item[$\bullet$] We model a distributed stochastic linear bandits (LBs) problem  where   $M$ agents  collaborate to minimize their total regret under the coordination of a central server when the contexts are observed with a delay and are unknown while choosing the action. We refer to this problem as the {\em distributed LBs with delayed contexts}.

\item[$\bullet$] We present a UCB-based algorithm that achieves a $O(d\sqrt{MT}\log^2 T)$ high probability regret bound for distributed LBs with delayed contexts.

\item[$\bullet$] We validated the performance of our approach via numerical simulations on synthetic data and on the real world Movielens data.
\end{enumerate}

The rest of the paper is organized as follows. In Section~\ref{sec:prob} we present the notations and the problem formulation. In Section~\ref{sec:rel} we present the related work.  In Section~\ref{sec:sol-2} we present the algorithm and  regret analysis.  In Section~\ref{sec:sim} we present the simulation results and in Section~\ref{sec:con} we present the conclusion.

\section{Problem Setting and Notations}\label{sec:prob}

\subsection{Notations}
The norm of a vector $z \in \Rb^d$ with respect to a matrix $V \in \Rb^{d \times d}$ is defined as $\|z \|_{V} : = \sqrt{z^\top V z}$ and $|z|$ for a vector $ z$ denotes element-wise absolute values. Further, $\top$ denotes matrix or vector transpose and $\langle \cdot, \cdot \rangle$ denotes inner product. For an integer $N$, we define $[N]:=\{1,2,\ldots,N\}$.

\subsection{Problem Setting: Distributed Linear Stochastic Bandits with Context Distribution}
In this section, we first specify the standard  linear bandit problem below and then explain the distributed stochastic  bandit setting studied in this paper.
Let $\X$ be the action set, $\C$ be the context set, and the environment is defined by a fixed and unknown reward  function $y: \X \times \C \rightarrow \mathbb{R}$.
 In linear bandit setting, at any time $t \in \mathbb{N}$, the agent observes a context $c_t \in \C$ and  has to choose an action $x_t \in \X$. Each context-action pair $(x,c)$, $x \in \X$ and $c \in \C$, is associated with a feature vector $\phi_{x,c} \in  \mathbb{R}^d$, i.e., $\pxc = \phi(x_t, c_t)$. Upon selection of an action $x_t$, the agent observes a  reward $y_t \in  \mathbb{R}$
\begin{equation}
y_t :=  \langle\theta^\star, \phi_{x_t, c_t}  \rangle + \eta_t,\label{eq:reward}
\end{equation}
where $\theta^\star \in \mathbb{R}^d$ is the unknown reward parameter, $ \langle\theta^\star, \phi_{x_t, c_t}  \rangle  = r(x_t, c_t)$ is the expected reward for action $x_t$ at time $t$, i.e., $r(x_t, c_t) = \mathbb{E}[y_t]$, and $\eta_t$ is  $\sigma-$subGaussian, additive  noise. The goal is to choose optimal actions $\xs_t$ for all $t \in T$ such that  the cumulative reward, $\sum_{t=1}^T y_t$, is maximized. This
is equivalent to minimizing the cumulative (pseudo)-regret denoted as 
 \begin{equation}
 \R_T = \sum_{t=1}^T\langle\theta^\star, \phi_{\xs_t, c_t}^t  \rangle - \sum_{t=1}^T\langle\theta^\star, \phi_{x_t, c_t}^t  \rangle.\label{eq:regret}
 \end{equation}
 Here $\xs_t$ is the optimal/best action for context $c_t$ and $x_t$ is the action chosen by the agent for context $c_t$. 

 In this work, we consider a {\em distributed stochastic} linear bandit setting with  context distribution and unknown contexts. The communication network consists of a server and a set of $M$ agents, and the agents can communicate with the server by sending and receiving packets. We assume that the communication between the server and the agents have zero latency.  We consider a setting where the context at time $t$, $c_t$ is {\em unobservable} rather only a distribution of the context denoted as $\mu_t$ is observed by the agents. At round $t$, the environment chooses a distribution $\mu_t \in \P(\C)$ over the context set and samples a context realization $c_t \sim \mu_t$.
The agents observe only $\mu_t$ and not  $c_t$ and each agent selects an action, say action chosen by agent $i$ is $\xti$, and receive reward $\yti$, where $\yti= \langle \thetas, \phi_{\xti, c_t} \rangle + \eta_{t, i}$.  
Our aim is to learn an optimal mapping/policy $\P(\C) \rightarrow \X$ of contexts  to actions  such that the cumulative reward, $ \sum_{i=1}^M \sum_{t=1}^T \yti$ is maximized. Formally, our aim is to minimize the cumulative regret 

 \begin{equation}
 \R(T) =  \sum_{i=1}^M \sum_{t=1}^T\langle\theta^\star, \phi_{\xs_{\ti}, c_t} \rangle - \sum_{i=1}^M  \sum_{t=1}^T\langle\theta^\star, \phi_{\xti, c_t} \rangle.\label{eq:regret-sc}
 \end{equation}
Here, $\xs_t = \arg\max_{x \in \X} \mathbb{E}_{c \sim \mu_t}[r_{x, c}]$ is the best action provided we know $\mu_t$, but not $c_t$, and $T$ is the total number of rounds.

Consider the  set $\I = \{(1, 2, \cdots , T) \times (1, 2, \cdots , M)\}  = \{\I_1, \I_2, \ldots, \I_{MT} \}$, which is the set of all possible $(t, i)$ pairs for $t \in [T]$ and $i \in [M]$. For $\I_\l =(t, i) \in \I$, we have $i = \l \bmod M$, $t =\lceil \l / M \rceil$, and we define $ \F_{\l} := \left\{ (x_{s,q}, \mu_{s}, y_{s, q} )\right\}_{\{(s, q):(s<t) \vee(s=t \wedge q<i)\}} $. 
We note that in \eqref{eq:regret-sc}, we compete with the best possible mapping $\pi^\star: \P(\C) \rightarrow \X$ from the observed context distribution to actions, that maximizes the expected reward $\sum_{i=1}^M\sum_{t=1}^T \mathbb{E}_{c_t \sim \mu_t}[r_{\pi^\star(\mu_t), c_t}|\F_{\l}, \mu_t]$, where $\F_{\l}$ is the filtration that contains all information available at the end of round $j-1$.


Our goal is to develop a distributed multi-armed bandit algorithm with the least possible communication cost to solve this problem. We define the communication cost of a protocol as the number of integers or real numbers communicated between the server and the  agents \cite{wang2019distributed}.
We make the standard assumptions on the additive noise $\eta_t$ and the unknown parameter $\theta^\star$ \cite{kirschner2019stochastic, kazerouni2016conservative}.
 
 \begin{assume}\label{assume:noise}
Each element $\eta_t$ of the noise sequence $\{\eta_t\}_{t=1}^{\infty}$ is conditionally $\sigma-$subgaussian, i.e., 
\begin{eqnarray*}
\mbox{for~all~} \lambda \in \Rb,~ \mathbb{E}[e^{\lambda \eta_t}|\F_{\l}, \mu_t] \geqslant \exp({\lambda^2 \sigma^2}/{2}). 
\end{eqnarray*}
\end{assume}
\begin{assume}\label{assume:bound}
There exist constants $S, D \geqslant 0$ such that $\norm{\theta^\star}_2 \leqslant S$, $\norm{\phi_{x,c_t}}_2 \leqslant D$, and $\phi_{x, c_t}^{\top}\theta^\star \in [0,1]$, for all $t$ and all $x \in \X$.
\end{assume}
\section{Related Work}\label{sec:rel}
MAB algorithms are well studied and various solution methods were suggested, for a survey see \cite{bubeck2012regret} and \cite{lattimore2020bandit}.  Our work deals with the class of linear contextual MABs with unknown contexts.  Linear contextual bandit problems with context-dependent uncertainty are studied in \cite{kirschner2019stochastic, lamprier2018profile, yun2017contextual, lin2022stochastic}. In \cite{yun2017contextual}, a scenario was explored in which contexts are disturbed by noise and the goal is to compete with the optimal policy that can access the undisturbed feature vector. Reference \cite{kirschner2019stochastic} studied a setting in which only a distribution on the context is known, as opposed to the exact context, and the goal is to pick the optimal action according to the distribution function. The model in \cite{kirschner2019stochastic} is closely related to ours, and the primary distinction is that while \cite{kirschner2019stochastic} considered a single-agent MAB problem we study a {\em multi-agent} MAB problem. In our initial work \cite{lin2022stochastic} we studied a single-agent conservative contextual MAB problem where the contexts are unknown and the learner is constrained to satisfy certain performance criteria. 

Multiple-player MAB has gained more attention recently \cite{yi2020distributed}. One class of problem study distributed MABs with collisions, where the reward for an arm reduces or is set to zero if a player chooses that action in \cite{bistritz2018distributed, kalathil2014decentralized, rosenski2016multi, anandkumar2011distributed}. In the following work of \cite{kalathil2014decentralized},   \cite{nayyar2016regret} investigated a context in which regret rises as a result of agent communication. A collision-based approach is associated with problems in cognitive radio networks, where the goal is to learn through action collisions rather than communication. This is in stark contrast to the setting considered in our work.  In \cite{wang2019distributed}, agents face the same bandit model and communicate with a central server by sending and acquiring information in order to learn concurrently and collaboratively.  Our model is similar to the  time-varying action set case considered in \cite{wang2019distributed}.  In our scenario, however, contexts are observed after a delay, whereas in \cite{wang2019distributed} the contexts are observed before choosing the action.

\section{Distributed UCB for Linear Stochastic  Bandits with Context Observation} \label{sec:sol-2}

\subsection{Proposed Algorithm and Guarantee}
In this section, we present our algorithm and regret bound for the setting where the actual context $c_t$ (e.g., actual weather measurements) is observable to the agents after they choose their actions. We note that with the context observation the agents have  $ \left\{ (x_{s,i}, c_{s}, y_{s, i} )\right\}_{s=1}^t$ available to them while estimating $\hat{\theta}_\ti$ although not for selecting the action. The  pseudocode of our algorithm is given in Algorithm~\ref{alg:TV-2}.

Given the distribution $\mu_t$, we first construct the feature vectors $\Psi_t = \{{ \psi}_{x, \mu_t}: x \in \X\}$,  where $\{{ \psi}_{x, \mu_t} := \mathbb{E}_{c \sim \mu_t}[\phi_{x,c}]\}$ is the  expected feature vector of action $x$ under $\mu_t$. Each feature ${ \psi}_{x, \mu_t}$ corresponds to exactly one action $x \in \X$ and $\Psi_t $ denotes the feature context set at time $t$. Algorithm~\ref{alg:TV-2} is based on the {\em optimism in the face of uncertainty} principle, where at each time $t \in [T]$,   each agent $i \in [M]$ maintains a confidence set $\B_{\ti} \subseteq \Rb^d$ that contains the unknown parameter vector $\theta^\star$ with high probability. Each agent then  chooses an optimistic estimate $\tilde{\theta}_{\ti} = \arg\max_{{\theta} \in \B_\ti}~(\max_{x \in \X}~ {\psi}_{x, \mu_t}^{\top}{\theta})$ and  chooses an action $x_\ti = \arg\max_{x \in \X}  {\psi}_{x, \mu_t}^{\top}\tilde{\theta}_\ti$. Equivalently the agent chooses the pair
$(x_\ti, \tilde{\theta}_\ti) \in \arg\max\limits_{(x,{\theta}) \in \X \times \B_\ti}  { \psi}_{x, \mu_t}^{\top} {\theta}$ which jointly maximizes the reward. The agents now play their respective optimistic actions, $x_\ti$'s, and receive  rewards $y_\ti$'s and   utilize the reward observations and the now observable context to update their individual confidence set. 

We note that while choosing the action the agents are unaware of the context and hence the decisions are made using $\psi$  rather than $\phi$ (line~\ref{line:12}). In line~\ref{line:13} $y_\ti$ is  a noisy observation of $\phi_{x_\ti, c_t}^\top \thetas$ and the algorithm expects the reward  $\psi_{x_\ti, \mu_t}^\top \thetas$. To address this,  we  construct a feature set $\Psi_t$ in such a way that $y_\ti$ is an unbiased observation for the action choice $\psi_t$,   similar to the technique in \cite{kirschner2019stochastic} for single agent bandits. After the actions are chosen, the agents receive the respective rewards and the contexts are observable now. Hence in the estimation we utilize the information about the context. We denote $\sum_{t}\phi_{\xti, \mu_{t}}\phi_{\xti, \mu_{t}}^{\top}$ and  $\sum_t \phi_{\xti, \mu_{t}}y_{\ti}$ for each agent $i \in [M]$ as $W_{\inew}$ and $U_{\inew}$, respectively. We construct the confidence set $\B_\ti$ using $W_{\inew}$ and $U_{\inew}$ as 
\begin{align}\label{eq:conf-set}
\B_\ti & = \Big\{\theta \in \Rb^d:\|\hat{\theta}_{\ti}-\theta\|_{\bar{V}_{\ti}} \leqslant \beta_{\ti} \Big\}, 
\end{align} 
where $\beta_\ti  = \beta_\ti(\sigma, \delta) = \sigma \, \sqrt{2\log\Big(\dfrac{\det(\bar{V}_\ti)^{1/2}\det(\lambda I)^{-1/2}}{\delta} \Big)}+\lambda^{1/2}S$, $\bar{V}_\ti = \lambda I + W_{\inew}$, $\hat{\theta}_\ti = \bar{V}_\ti^{-1}U_{\inew}$.

The agents in our protocol share their local estimates with the central server during the synchronization step. The synchronizations are done at specific time instants. We refer to the timesteps between the two synchronizations as {\em epochs}.   The epochs are designed  based on the  observation  in \cite{abbasi2011improved} that the change in the determinant of $\bar{V}_t$ is a good indicator of learning progress. Based on this observation, we
only synchronize when agent $i$ finds that the log-determinant of $\bar{V}_\ti$ has changed more than a constant factor since
the last synchronization, and this reduces the communication cost of the algorithm.  
\begin{algorithm}[h!]
\caption{Distributed UCB for LBs with delayed contexts}\label{alg:TV-2}
\begin{algorithmic}[1] 
\State  \textit {\bf  Initialization:} $ B = (\frac{T \log MT}{d M}) $, $\lambda = 1$, $W_{\isyn}=0, U_{\isyn}=0, W_{\inew}=0, U_{\inew}=0, t_{last}=0, V_{last}=\lambda I$, for all $\ i=1, 2, \ldots, M$
\For{$t=1,2,\ldots, T$}
\State Nature chooses $\mu_t \in \P(\C) $
\State Learner observes $\mu_t$
\State Set $\Psi_t = \{{ \psi}_{x, \mu_t}: x \in \X\}$ where $\{{ \psi}_{x, \mu_t} := \mathbb{E}_{c \sim \mu_t}[\phi_{x,c}]\}$\label{step:psi}
\For{Agent $\ i=1,2,\ldots,M, $}
\State $\overline{V}_{\ti}=\lambda I+W_{\isyn}+W_{\inew} , \hat{\theta}_{\ti}=\bar{V}^{-1}_{\ti}(U_{\isyn}+U_{\inew})$
\State Construct the confidence ellipsoid $\B_{\ti}$ using $\overline{V}_{\ti}$ and $ \hat{\theta}_{\ti}$
\State $(x_{\ti}, \tilde{\theta}_{\ti})=\mathop{\arg\max}_{(x, \theta) \in \X \times \B_{\ti}} \left\langle \psi_{x, \mu_t}, \theta \right\rangle$\label{line:12}
\State Play $x_{\ti}$ and get the reward $y_{\ti}$  \label{line:13}
\State Update $W_{\inew}=W_{\inew}+\phi_{\xti, \mu_{t}}\phi_{\xti, \mu_{t}}^{\top}, U_{\inew}=U_{\inew}+\phi_{\xti, \mu_{t}}y_{\ti}$\label{line:update}
\State $V_{\ti} = \lambda I +W_{\isyn}+W_{\inew}$
\If {$\log(\det(V_{\ti})/\det(V_{last}))\cdot(t-t_{last}) \geqslant B$}
\State Send a synchronization signal to server to start a communication round
\EndIf
\State {\bf Synchronization round:}
\If {a communication round is started}
\State All agents $i \in [M]$ send $W_{\inew}$ and $U_{\inew}$ to server
\State Server computes $W_{\mathrm{syn}}=W_{\mathrm{syn}}+\sum_{i=1}^{M}W_{\inew}, U_{\mathrm{syn}}=U_{\mathrm{syn}}+\sum_{i=1}^{M}U_{\inew}$
\State All agents receive $W_{\mathrm{syn}}, U_{\mathrm{syn}}$ from the server
\State  Set $ W_{\inew}=0, U_{\inew}=0, t_{last}=t, $ for all $i$, $ V_{last}=\lambda I+W_{\mathrm{syn}} $
\EndIf
\EndFor
\EndFor
\end{algorithmic}
\end{algorithm}
\begin{theorem}\label{thm:main-2}
The cumulative regret of Algorithm~\ref{alg:TV-2} with expected feature set $\Psi_t$ and $\beta_\ti = \beta_\ti(\sigma, \delta/3)$  is bounded at time $T$ with probability (w. p.) at least $1-M\delta$ by
\begin{align*}
\R(T) \hspace*{-1 mm} &\leqslant \hspace*{-1 mm} 4\beta_T \sqrt{MT d \log (MT)}(1\hspace*{-1 mm}+\hspace*{-1 mm}\log(MT)) \hspace*{-0.75 mm}+ \hspace*{-0.75 mm}4 (\beta_T\hspace*{-0.75 mm}+\hspace*{-0.75 mm}1) \sqrt{2MT\log \frac{3}{\delta}}. 
\end{align*}
Further, for $\delta=\dfrac{1}{M^2T}$, Algorithm~\ref{alg:TV-2} achieves a regret of $ O(d \sqrt{M T} \log ^{2}(T)) $ with $ O(M^{1.5} d^{3}) $ communication cost. 
\end{theorem}
\begin{proof}
See Section~\ref{sec:regret-analysis2}.
\end{proof}
The significance of Theorem~\ref{thm:main-2} is that it allows us to use a smaller scaling $\beta_\ti$ for the confidence set, which indeed affects the action chosen by the algorithm. It is known that in practice $\beta_\ti$ has a large impact on the amount of the exploration, and a tighter choice of $\beta_\ti$ can result in a significant reduction of the regret bound \cite{kirschner2019stochastic}, which we validate through the experiments.
\subsection{Regret Analysis}\label{sec:regret-analysis2}
Proof of  Theorem~\ref{thm:main-2} relies on Lemma~\ref{LT1}, Proposition~\ref{LT2}, and  the two main results we present below,  Theorems~\ref{thm:good-2} and~\ref{thm:bad-2}. 

\begin{lemma} \label{LT1}
For any $\delta > 0$ w.p. $1-M\delta,$ $\theta^{\star}$ always lies in the constructed $\B_{\ti}$ for all $t$ and $i$. 
\end{lemma}
\begin{proof}
The proof follows using Theorem~2 in \cite{abbasi2011improved} and union bound over all agents. 
\end{proof}
For a positive definite matrix $V \in \Rb^{d \times d}$,  we have the following result from \cite{abbasi2011improved}.
\begin{prop}[Lemma~11, \cite{abbasi2011improved}] \label{LT2}
Let $ \left\{X_t\right\}_{t=1}^{\infty} $ be a sequence in $ \mathbb{R}^{d}, V $ is a $ d \times d $ positive definite matrix and define $ \bar{V}_{t}=V+\sum_{s=1}^{t} X_s X_s^{\top} $. We have that
\[
\log \left(\frac{\det\left(\bar{V}_{n}\right)}{\det(V)}\right) \leqslant \sum_{t=1}^{n}\left\|X_t\right\|^2_{\bar{V}_{t-1}^{-1}}.\]
Further, if $ \left\|X_t\right\|_{2} \leqslant L $ for all $ t $, then
\[
\sum_{t=1}^{n} \min \left\{1,\left\|X_t\right\|_{V_{t-1}^{-1}}^{2}\right\}  \leqslant  2\left(\log \det\left(\bar{V}_{n}\right)-\log \det (V)\right)\]
\[\leqslant  2\left(d \log \left(\left(\trace(V)+n L^{2}\right) / d\right)-\log \det (V)\right).
\]
\end{prop}
The lemma below proves a bound on the per-step regret of the protocol.
\begin{lemma}\label{lem:observation}
In Algorithm~\ref{alg:TV-2}, with probability $ 1- \delta $, the single step pseudo-regret $ r_{\ti}=\left\langle\theta^{\star}, \phi_{x_\ti^{\star}, c_{t}}-\phi_{x_{\ti}, c_{t}}\right\rangle $ with $\beta_\ti = \beta_\ti(\sigma, \delta)$ is bounded by
\begin{align*}
r_{\ti} & \leqslant 2 \beta_{\ti} (\left\| \phi_{x_{\ti}, c_{t}} \right\|_{\bar{V}_{\ti}^{-1}} + \S_{j} ) + \D_{j}
\end{align*}
where $ \D_{j} = \left\langle\theta^{\star}, \phi_{x_\ti^{\star}, c_{\htt}} - \phi_{x_{\htt, \hii}, c_{\htt}} - \psi_{x_\ti^{\star}, \mu_{\htt}} + \psi_{x_{\htt, \hii}, \mu_{\htt}}\right\rangle$ and $ \S_{j} = \| \psi_{x_{\ti}, \mu_{t}} \|_{\bar{V}_{\ti}^{-1}} - \| \phi_{x_{\ti}, c_{t}} \|_{\bar{V}_{\ti}^{-1}}$.
\end{lemma}
\begin{proof}
Let us assume that $ \theta^{\star} \in \B_{\ti}$. Then we have

\begin{align}
r_{\ti} &=\left\langle\theta^{\star}, \phi_{x_\ti^{\star}, c_{t}}\right\rangle-\left\langle\theta^{\star}, \phi_{x_{\ti}, c_{t}}\right\rangle \nonumber \\
&=\langle\theta^{\star}, \psi_{x_\ti^{\star}, \mu_{t}}\rangle + \langle\theta^{\star}, \phi_{x_{\ti}^{\star}, c_{t}} - \psi_{x_{\ti}^{\star}, \mu_{t}}\rangle\nonumber\\
& - \langle\theta^{\star}, \psi_{x_{\ti}, \mu_{t}}\rangle - \langle\theta^{\star}, \phi_{x_{\ti}, c_{t}} - \psi_{x_{\ti}, \mu_{t}}\rangle \nonumber \\
& \leqslant \langle\ \tilde{\theta}_{\ti}, \psi_{x_{\ti}, \mu_{t}}\rangle - \langle\theta^{\star}, \psi_{x_{\ti}, \mu_{t}}\rangle\nonumber\\
& + \langle\theta^{\star}, \phi_{x_{\ti}^{\star}, c_{t}} - \phi_{x_{\ti}, c_{t}} - \psi_{x_{\ti}^{\star}, \mu_{t}} + \psi_{x_{\ti}, \mu_{t}}\rangle \nonumber \\
&= \langle\ \tilde{\theta}_{\ti} - \hat{\theta}_{\ti}, \psi_{x_{\ti}, \mu_{t}}\rangle + \langle\ \hat{\theta}_{\ti} - \theta^{\star}, \psi_{x_{\ti}, \mu_{t}}\rangle\nonumber\\
& + \langle\theta^{\star}, \phi_{x_{\ti}^{\star}, c_{t}} - \phi_{x_{\ti}, c_{t}} - \psi_{x_{\ti}^{\star}, \mu_{t}} + \psi_{x_{\ti}, \mu_{t}}\rangle \nonumber \\
&\leqslant \| \tilde{\theta}_{\ti} - \hat{\theta}_{\ti}\| \| \psi_{x_{\ti}, \mu_{t}} \| + \| \hat{\theta}_{\ti} - \theta^{\star}\| \| \psi_{x_{\ti}, \mu_{t}} \| + \D_{j} \nonumber\\
 &\leqslant \| \ols{V}_{\ti}^{-1/2} \left( \tilde{\theta}_{\ti} - \hat{\theta}_{\ti} \right) \|_{\ols{V}_{\ti}} \| \ols{V}_{\ti}^{1/2} \psi_{x_{\ti}, \mu_{t}} \|_{\ols{V}_{\ti}^{-1}} \nonumber\\
& + \| \ols{V}_{\ti}^{-1/2} \left( \hat{\theta}_{\ti} - \theta^{\star} \right) \|_{\ols{V}_{\ti}} \| \ols{V}_{\ti}^{1/2} \psi_{x_{\ti}, \mu_{t}} \|_{\ols{V}_{\ti}^{-1}} + \D_{j} \nonumber \\
&\leqslant \| \ols{V}_{\ti}^{-1/2} \|_{\ols{V}_{\ti}} \| \tilde{\theta}_{\ti} - \hat{\theta}_{\ti} \|_{\ols{V}_{\ti}} \| \ols{V}_{\ti}^{1/2} \|_{\ols{V}_{\ti}^{-1}} \| \psi_{x_{\ti}, \mu_{t}} \|_{\ols{V}_{\ti}^{-1}} \nonumber \\
&+ \| \ols{V}_{\ti}^{-1/2} \|_{\ols{V}_{\ti}} \|  \hat{\theta}_{\ti} - \theta^{\star} \|_{\ols{V}_{\ti}} \| \ols{V}_{\ti}^{1/2} \|_{\ols{V}_{\ti}^{-1}} \| \psi_{x_{\ti}, \mu_{t}} \|_{\ols{V}_{\ti}^{-1}} + \D_{j} \nonumber\\
&= (\| \tilde{\theta}_{\ti} - \hat{\theta}_{\ti} \|_{\ols{V}_{\ti}} + \|  \hat{\theta}_{\ti} - \theta^{\star} \|_{\bar{V}_{\ti}}) \| \psi_{x_{\ti}, \mu_{t}}\|_{\bar{V}_{\ti}^{-1}} + \D_{j} \nonumber\\
&\leqslant 2 \beta_{\ti} \| \psi_{x_{\ti}, \mu_{t}} \|_{\bar{V}_{\ti}^{-1}} + \D_{j} \label{s3:eq1}
\end{align}
\begin{align}
&= 2 \beta_{\ti} (\| \phi_{x_{\ti}, c_{t}} \|_{\bar{V}_{\ti}^{-1}} + \| \psi_{x_{\ti}, \mu_{t}} \|_{\bar{V}_{\ti}^{-1}} - \| \phi_{x_{\ti}, c_{t}} \|_{\bar{V}_{\ti}^{-1}} ) + \D_{j} \nonumber \\ 
& \leqslant 2 \beta_{\ti} (\left\| \phi_{x_{\ti}, c_{t}} \right\|_{\bar{V}_{\ti}^{-1}} + \S_{j} ) + \D_{j}.
\end{align}
Eq.~\eqref{s3:eq1} follows from Lemma~\ref{LT1}, where $ \| \hat{\theta}_{\ti} - \theta \|_{\bar{V}_{\ti}} \leqslant \sqrt{2 \log(\frac{\det(\bar{V}_{\ti})^{1/2}\det(\lambda I)^{-1/2}}{\delta})}+\lambda^{1/2} = \beta_{\ti}$.
\end{proof}
Below we present the Azuma-Hoeffdings inequality. 
\begin{prop} \label{prop:Azu}
(Azuma-Hoeffdings) Let $ M_{j} $ be a martingale on a filtration $ \F_{j} $ with almost surely bounded increments $ \left|M_{j}-M_{j-1}\right|<Q $. Then
\[
\mathbb{P}\left[M_{n}-M_{0}>b\right] \leqslant \exp \left(-\frac{b^{2}}{2 n Q^{2}}\right).
\]
\end{prop}
\begin{lemma}\label{lem:Dj-bound}
$ \D_{j} := \left\langle\theta^{\star}, \phi_{x_\ti^{\star}, c_{\htt}} - \phi_{x_{\htt, \hii}, c_{\htt}} - \psi_{x_\ti^{\star}, \mu_{\htt}} + \psi_{x_{\htt, \hii}, \mu_{\htt}}\right\rangle
$ is a martingale difference sequence with $ \left|\D_{j}\right| \leqslant 4 $ and $ \sum_{j} \D_{j} $ is a martingale. Further,  $ \sum_{j } \D_{j} $ is bounded w.p. at least  $1-\delta$  as  \[\sum_{j} \D_{j} \leqslant 4 \sqrt{2n \log \frac{1}{\delta}}.\]
\end{lemma}
\begin{proof}
Recall that for $ \D_{j} = \langle\theta^{\star}, \phi_{x_\ti^{\star}, c_{\htt}} - \phi_{x_{\ti}, c_{\htt}} - \psi_{x_\ti^{\star}, \mu_{\htt}} + \psi_{x_{\ti}, \mu_{\htt}}\rangle$, $\I_j=(\ti)$ with $\hii= \l \bmod M$, $\htt =\lceil \l / M \rceil$, and $ \F_{\l} = \left\{ (x_{s,q}, \mu_{s}, y_{s, q} )\right\}_{\{(s, q):(s<t) \vee(s=t \wedge q<i)\}} $. Thus we have 
{\scalefont{0.85}
\begin{align}
&\mathbb{E}_{c_{t} \sim \mu_{t}} [ \D_{j} \mid \F_{\l}, \mu_{\htt}, x_{\htt, \hii} ] \nonumber\\
&= \mathbb{E}_{c_{t} \sim \mu_{t}} [\langle\theta^{\star}, \phi_{x_\ti^{\star}, c_{\htt}} - \phi_{x_{\htt, \hii}, c_{\htt}} - \psi_{x_\ti^{\star}, \mu_{\htt}} + \psi_{x_{\htt, \hii} , \mu_{\htt}} \rangle \mid \F_{\l}, \mu_{\htt}, x_{\htt, \hii} ]\nonumber \\
&= \mathbb{E}_{c_{t} \sim \mu_{t}} [ \langle \theta^{\star}, \phi_{x_\ti^{\star}, c_{\htt}} \rangle \mid \F_{\l}, \mu_{\htt}, x_{\htt, \hii}   ]\nonumber\\
& - \mathbb{E}_{c_{t} \sim \mu_{t}} [ \langle \theta^{\star}, \phi_{x_{\htt, \hii}, c_{\htt}} \rangle \mid \F_{\l}, \mu_{\htt}, x_{\htt, \hii} ] \nonumber \\
&- \mathbb{E}_{c_{t} \sim \mu_{t}} [ \langle \theta^{\star}, \psi_{x_\ti^{\star}, \mu_{\htt}} \rangle \mid \F_{\l}, \mu_{\htt}, x_{\htt, \hii} ]\nonumber\\
& + \mathbb{E}_{c_{t} \sim \mu_{t}} [ \langle \theta^{\star}, \psi_{x_{\htt, \hii}, \mu_{\htt}} \rangle \mid \F_{\l}, \mu_{\htt}, x_{\htt, \hii} ] \nonumber\\
&= \langle \theta^{\star}, \psi_{x_\ti^{\star}, \mu_{\htt}} \rangle - \langle \theta^{\star}, \psi_{x_{\htt, \hii}, \mu_{\htt}} \rangle - \langle \theta^{\star}, \psi_{x_\ti^{\star}, \mu_{\htt}} \rangle + \langle \theta^{\star}, \psi_{x_{\htt, \hii}, \mu_{\htt}} \rangle = 0. \label{c_mu}
\end{align}
}
Eq.~\eqref{c_mu} follows from $ \psi_{x_{\htt, \hii} , \mu_{\htt}}=\mathbb{E}_{c_{\htt} \sim \mu_{\htt}}\left[\phi_{x_{\htt, \hii}, c_{\htt}} \mid \F_{\l}, \mu_{\htt}, x_\ti\right]$. Therefore $ \D_{j} $ is a martingale difference sequence with $ \left|\D_{j}\right| \leqslant 4 $ and $ \sum_{j =1}^{MT} \D_{j} $ is a martingale.  Using Proposition~\ref{prop:Azu} with $ Q=4 $ 
$
 \sum_{j} \D_{j} \leqslant  4 \sqrt{2n \log \frac{1}{\delta}}.
$
\end{proof}

 %
 %
  Below we prove $\sum_{j\in [MT]}\S_j$ is a supermartingale.
\begin{lemma}\label{lem:Sj}
Let  $ \S_{j} = \| \psi_{x_{\ti}, \mu_{t}} \|_{\bar{V}_{\ti}^{-1}} - \| \phi_{x_{\ti}, c_{t}} \|_{\bar{V}_{\ti}^{-1}}$. Then,  $ \sum_{j =1}^{MT} \S_{j} $ is a supermartingale with $ |\S_{j}| = 2 \lambda^{-1/2} $. Further, $ \sum_{j} \S_{j} $ is bounded with probability at least  $1-\delta$ as  \[\sum_{j} \S_{j} \leqslant 2\lambda^{-1/2} \sqrt{2 n \log \frac{1}{\delta}}.\]
\end{lemma}
\begin{proof}
Recall that $ \S_{j} = \| \psi_{x_{\ti}, \mu_{t}} \|_{\bar{V}_{\ti}^{-1}} - \| \phi_{x_{\ti}, c_{t}} \|_{\bar{V}_{\ti}^{-1}}$, $\I_j=(\ti)$  with $\hii= \l \bmod M$, $\htt =\lceil \l / M \rceil$, and $ \F_{\l} := \left\{ (x_{s,q}, \mu_{s}, y_{s, q} )\right\}_{\{(s, q):(s<t) \vee(s=t \wedge q<i)\}} $. Thus we have
\begin{align}
&\mathbb{E}_{c_{t} \sim \mu_{t}} [ \S_{j} \mid \F_{\l}, \mu_{t}, x_\ti ]\nonumber\\
 &= \mathbb{E}_{c_{t} \sim \mu_{t}} [\| \psi_{x_{\ti}, \mu_{t}} \|_{\bar{V}_{\ti}^{-1}}- \| \phi_{x_{\ti}, c{t}} \|_{\bar{V}_{\ti}^{-1}} \mid \F_{\l}, \mu_{t}, x_\ti] \nonumber\\
&= \mathbb{E}_{c_{t} \sim \mu_{t}} [\| \psi_{x_{\ti}, \mu_{t}} \|_{\bar{V}_{j}^{-1}} \mid \F_{\l}, \mu_{t}, x_\ti] \nonumber\\
& - \mathbb{E}_{c_{t} \sim \mu_{t}} [\| \phi_{x_{\ti}, c{t}} \|_{\bar{V}_{\ti}^{-1}} \mid \F_{\l}, \mu_{t}, x_\ti] \nonumber 
\end{align}
\begin{align}
&\leqslant \left\| \psi_{x_{\ti}, \mu_{t}} \right\|_{\bar{V}_{\ti}^{-1}} - \left\| \psi_{x_{\ti}, \mu_{t}} \right\|_{\bar{V}_{\ti}^{-1}} = 0 \label{s3:eq10}
\end{align}
Eq.~\eqref{s3:eq10} follows from $ \psi_{x_{\ti}, \mu_{t}} = \mathbb{E}_{c_{t} \sim \mu_{t}} [\phi_{x_{\ti}, c_{t}} \mid \F_{\l}, \mu_{t}, x_\ti] $, and $ \| \psi_{x_{\ti}, \mu_{t}} \|_{\bar{V}_{\ti}^{-1}} = \| \mathbb{E}_{c_{t} \sim \mu_{t}} [\phi_{x_{\ti}, c_{t}} \mid \F_{\l}, \mu_{t}, x_\ti] \|_{\bar{V}_{\ti}^{-1}} \leqslant \mathbb{E}_{c_{t} \sim \mu_{t}} [\| \phi_{x_{\ti}, c_{t}} \|_{\bar{V}_{\ti}^{-1}} \mid \F_{\l}, \mu_{t}, x_\ti] $ since $ \varphi (\mathbb{E} [X]) \leqslant \mathbb{E} [\varphi (x)] $  by  Jensen's inequality. 
\begin{align}
|\S_{j}| &= |\| \psi_{x_{\ti}, \mu_{t}} \|_{\bar{V}_{\ti}^{-1}} - \| \phi_{x_{\ti}, c_{t}} \|_{\bar{V}_{\ti}^{-1}} | \leqslant | \| \psi_{x_{\ti}, \mu_{t}} \|_{\bar{V}_{\ti}^{-1}} | + | \| \phi_{x_{\ti}, c_{t}} \|_{\bar{V}_{\ti}^{-1}} | \nonumber \\
&\leqslant | \mathbb{E}_{c_{t} \sim \mu_{t}} [\| \phi_{x_{\ti}, c_{t}} \|_{\bar{V}_{j}^{-1}} \mid \F_{\l}, \mu_{t}, x_\ti] | + | \lambda^{-1/2} | \label{s3:eq11} \\
&\leqslant | \lambda^{-1/2} | + | \lambda^{-1/2} | = 2 \lambda^{-1/2} \nonumber 
\end{align}
Eq.~\eqref{s3:eq11} follows from $ \psi_{x_\ti, \mu_{t}}=\mathbb{E}_{c_t \sim \mu_{t}}[\phi_{x_\ti, c_t} \mid \F_{\l}, \mu_{t}, x_\ti]$ and 
\begin{align}
\| \phi_{x_{\ti}, c_{t}} \|_{\bar{V}_{\ti}^{-1}} &= \sqrt{\phi_{x_{\ti}, c_{t}}^{\top} \bar{V}_{\ti}^{-1/2} \bar{V}_{\ti}^{-1/2} \phi_{x_{\ti}, c_{t}}} \nonumber \\
&= \| \phi_{x_{\ti}, c_{t}}^{\top} \bar{V}_{\ti}^{-1/2} \| \leqslant \| \phi_{x_{\ti}, c_{t}}^{\top} \| \| \bar{V}_{\ti}^{-1/2} \| \nonumber \\
&= \| \phi_{x_{\ti}, c_{t}}^{\top} \| \| (\lambda I + \sum_{t, \, i} \phi_{x_{\ti}, c_{t}} \phi_{x_{\ti}, c_{t}}^{\top})^{-1/2} \| \nonumber \\
&\leqslant \| \phi_{x_{\ti}, c_{t}}^{\top} \| \| ( \lambda I )^{-1/2} \| \leqslant \lambda^{-1/2}. \label{s3:eq12}
\end{align}
Eq.~\eqref{s3:eq12} follows from $ \bar{V}_{\ti}^{-1}  \leqslant \bar{V}_{t-1, i}^{-1}$.  Therefore $ \sum_{j =1}^{MT} \S_{j} $ is a supermartingale with $ |\S_{j}| = 2 \lambda^{-1/2} $.

Now from Proposition~\ref{prop:Azu} we have
\[
 \sum_{j} \S_{j} \leqslant  2\lambda^{-1/2} \sqrt{2 n \log \frac{1}{\delta}}.
\]
\end{proof}

Consider an arbitrary epoch  in Algorithm~\ref{alg:TV-2}, say the $p^{\rm th}$ epoch. Let $\E_p$ be the set of all $(t, i)$ pairs in epoch $p$ and $V_{p}$ be the $ V_{l a s t} $ in epoch $p$. Then we know
\[
 \frac{{\mathrm{det}(V_{p})}}{{\mathrm{det} (V_{p-1})}} = 1+ \sum_{(t,i)\in \A_p}\|\psit \|^2_{\bar{V}_{t-1, i}^{-1}}.
\]
 Assume  $\sum_{(t,i)\in \A_p}\|\psit \|^2_{\bar{V}_{t-1, i}^{-1}}\leqslant 1$. Then we have
\begin{equation}
1   \leqslant \frac{\mathrm{det}(V_{p})}{\mathrm{det} (V_{p-1})} \leqslant 2. \label{E16} 
\end{equation}
All the epochs that satisfy  \eqref{E16} are referred to as the {\em good} epochs. Similarly,  all the epochs that do not satisfy \eqref{E16} are referred to as the {\em bad} epochs. Our approach to prove  Theorem~\ref{thm:main-2} is to prove bounds for good epochs and bad epochs separately as described below.  Let us denote the  number of timesteps that belong to good epochs as $T_\good$ and the  number of  timesteps that belong to  bad epochs as $T_\bad$. 
 Let us denote the cumulative regret in all good epochs until time $T$ as $\R_{\good}(T) $ and the cumulative regret in all bad epochs until time $T$  as $\R_{\bad}(T)$. We present bounds for  $\R_{\good}(T) $  and $\R_{\bad}(T) $ separately and then use those bounds  to prove Theorem~\ref{thm:main-2}.
Our approach  uses similar argument in Theorem~4 in \cite{abbasi2011improved}.
\begin{theorem}\label{thm:good-2}
The cumulative regret of all good epochs in Algorithm~\ref{alg:TV-2} with expected feature set $\Psi_t$ and $\beta_\ti = \beta_\ti(\sigma, \delta/3)$  is bounded at time $T$ with probability at least $1-M\delta$ by
\[
\Rgood\leqslant 4\beta_T\sqrt { MT_\good d\log(MT)}+4(\beta_T+1)\sqrt{2MT_\good \log \frac{3}{\delta}}.  
\]
\end{theorem}
\begin{proof}
To bound the cumulative regret of good epochs we  use Theorem~4 in \cite{abbasi2011improved}. Assume that the $ MT $ pulls are all made by one agent in a round-robin fashion (i.e., the agent takes $ x_{1,1}, x_{1,2}, \ldots, x_{1, M}, x_{2,1}, \ldots, x_{2, M}, \ldots, x_{T, 1}, \ldots,  x_{T, M} $). We define $\widetilde{V}_{\ti} := \lambda I + \sum_{\{(p, q):(p<t) \vee(p=t \wedge q<i)\}} x_{p, q} x_{p, q}^{\top}$. Thus $ \widetilde{V}_{\ti}$  denotes $\bar{V}_{\ti} $ the imaginary agent calculates when the agent gets to $ x_{\ti} $. Since $ x_{\ti} $ is in a good epoch (say the $p$-th epoch), we have
\begin{align}
1 \leqslant \frac{\det (\widetilde{V}_{\ti})}{\mathrm{det} (\bar{V}_{\ti})} \leqslant \frac{\mathrm{det}(V_{p})}{\mathrm{det} (V_{p-1})} \leqslant 2 \label{s3:eq3} 
\end{align}
Eq.~\eqref{s3:eq3} uses $\mathrm{det}(\bar{V}_{\ti}) \geqslant \mathrm{det}(V_{p-1})$ and the fact that for good epochs $\mathrm{det}(\widetilde{V}_{\ti}) =\mathrm{det}(V_{p})$. Thus from Lemma~\ref{lem:observation} we have
\begin{align}
r_{\ti} & \leqslant 2 \beta_{\ti} \sqrt{ \phi_{x_{\ti}, c_{t}}^{\top} \bar{V}_{\ti}^{-1} \phi_{x_{\ti}, c_{t}}} + 2 \beta_{\ti} \S_{j} + \D_{j} \nonumber \\
& = 2 \beta_{\ti} \sqrt{\phi_{x_{\ti}, c_{t}}^{\top} \widetilde{V}_{\ti}^{-1} \phi_{x_{\ti}, c_{t}} \frac {\phi_{x_{\ti}, c_{t}}^{\top} \bar{V}_{\ti}^{-1} \phi_{x_{\ti}, c_{t}}} {\phi_{x_{\ti}, c_{t}}^{\top} \widetilde{V}_{\ti}^{-1} \phi_{x_{\ti}, c_{t}}}} + 2 \beta_{\ti} \S_{j} + \D_{j} \nonumber \\
& = 2 \beta_{\ti} \sqrt{\phi_{x_{\ti}, c_{t}}^{\top} \widetilde{V}_{\ti}^{-1} \phi_{x_{\ti}, c_{t}} \cdot \frac{\mathrm{det} (\widetilde{V}_{\ti})}{\mathrm{det} (\bar{V}_{\ti})}} + 2 \beta_{\ti} \S_{j} + \D_{j} \nonumber \\
& \leqslant 2 \beta_{\ti} \sqrt{2 \phi_{x_{\ti}, c_{t}}^{\top} \widetilde{V}_{\ti}^{-1} \phi_{x_{\ti}, c_{t}}} + 2 \beta_{\ti} \S_{j} + \D_{j} \label{s3:eq4} 
\end{align}
The last three steps follow from $\frac {x^{\top} A x} {x^{\top} B x} \leqslant \frac {\mathrm{det}(A)} {\mathrm{det}(B)}$ and \eqref{s3:eq3}.
We now use the argument for the single agent regret bound and prove regret bound for the good epochs.
Recall that $\E_{p}$ denotes the set of $ (t, i) $ pairs that belong to epoch $p$,  $ P_{\good} $ denotes the set of good epochs, and $\A_\good$ denotes the set of all $j \in [MT]$ that belong to good epochs. Using Proposition~\ref{LT2}, we bound 
\begin{align*}
\Rgood- \sum_{p \in P_{\good}}  \sum\limits_{\substack{(\ti) \in \E_{p}\\j:\I_j=(\ti)}} \Big(2\beta_\ti \S_{j} + \D_j\Big)
\end{align*}
\begin{align}
& = \sum_{p \in P_{good}}  \sum\limits_{\substack{(\ti) \in \E_{p}\\j:\I_j=(\ti)}} \Big(r_{\ti} - 2\beta_\ti\S_{j} -  \D_{j}\Big) \nonumber \\
&= \sqrt{\Big(\sum_{p \in P_{\good}}  \sum\limits_{\substack{(\ti) \in \E_{p}\\j:\I_j=(\ti)}} \Big(r_{\ti} - 2\beta_\ti\S_{j} -  \D_{j}\Big)\Big)^{2}} \nonumber 
\end{align}
\begin{align}
& \leqslant \sqrt{MT_\good \sum_{p \in P_{\good}}\sum\limits_{\substack{(\ti) \in \E_{p}\\j:\I_j=(\ti)}} \Big(r_{\ti} - 2\beta_\ti\ \S_{j} -  \D_{j}\Big)^{2}} \label{s3:eq5} \\
& \leqslant 2\beta_T \sqrt{ 2MT_\good \sum_{p \in P_{\good}} \sum_{(t, i) \in \E_{p}} \min ( \left\| \phi_{x_{\ti}, c_{t}} \right\|_{\widetilde{V}_{\ti}^{-1}}^{2}, 1 )}  \label{s3:eq6}\\
& \leqslant 2\beta_T \sqrt{ 2MT_\good\sum_{p \in P_{\good}} \sum_{(t, i) \in \E_{p}} \min ( \left\| \phi_{x_{\ti}, c_{t}} \right\|_{\widetilde{V}_{t-1, i}^{-1}}^{2}, 1 )} \label{s3:eq7}  \\
& \leqslant 2\beta_T \sqrt{ 2MT_\good \sum_{p \in P_{\good}} \log \Big(\frac{\mathrm{det} (V_{p} )}{\mathrm{det} (V_{p-1} )} \Big)} \label{s3:eq8} \\
& \leqslant 2\beta_T \sqrt { 2MT_\good \log \Big(\frac{\mathrm{det} (V_{P} )}{\mathrm{det} (V_{0} )} \Big)} \leqslant 4\beta_T \sqrt { MT_\good d  \log(MT)}  \label{s3:eq9}
\end{align}
In the  inequalities above, \eqref{s3:eq5} uses $(\sum_{i=1}^{n} z_{i})^{2} \leqslant n \sum_{i=1}^{n} z_{i}^{2}$ for $z_i \in \mathbb{R}$,  \eqref{s3:eq6} follows from \eqref{s3:eq4},  \eqref{s3:eq7} follows from $ \bar{V}_{\ti}^{-1} \leqslant \bar{V}_{i, t-1}^{-1}$, and \eqref{s3:eq8} follows from Proposition~\ref{LT2}. Finally,  \eqref{s3:eq9} follows from $
\log (\frac{\mathrm{det} (V_{P} )}{\mathrm{det} (V_{0} )} )  \leqslant d \log (MT)^{2}.$

Now from Lemma~\ref{lem:Sj} ($\lambda=1$) and Lemma~\ref{lem:Dj-bound} with $ n = MT_\good $ and after applying a union bound, we can rewrite  \eqref{s3:eq9} as below with $\beta_\ti = \beta_\ti(\sigma, \delta/3)$
%
\begin{align}
\Rgood 
&\leqslant  4\beta_T \sqrt { MT_\good d  \log(MT)}  + 4(\beta_T+1)\sqrt{2MT_\good \log \frac{3}{\delta}}.\nonumber
\end{align}
%
\end{proof}
\begin{theorem}\label{thm:bad-2}
The cumulative regret of all bad epochs in Algorithm~\ref{alg:TV-2} with expected feature set $\Psi_t$ and $\beta_\ti = \beta_\ti(\sigma, \delta/3)$  is bounded at time $T$ with probability at least $1-M\delta$ by
\[
\Rbad\leqslant 4d\log(MT)(\beta_T M\sqrt{B})+4(\beta_T+1) \sqrt{2MT_\bad\log \frac{3}{\delta}}.
\]
\end{theorem}
\begin{proof}
Let $p$ be a bad epoch. Assume that  a bad  epoch  starts at time step $ t_{0} $ and that the length of the epoch is $ n $. Then agent $ i $ proceeds as $ \bar{V}_{t_{0}, i}, \ldots, \bar{V}_{t_{0}+n, i} $. Using Lemma~\ref{lem:observation} regret in this epoch, denoted as $\R_p$, satisfies
\begin{align*}
\R_p - \sum\limits_{\substack{ (\ti) \in \E_{p}, \\j:\I_j=(\ti)}} \beta_\ti\S_{j} - \sum\limits_{\substack{j : \I_j=(\ti)\\ \& (\ti) \in \E_{p}}}\D_{j}
\end{align*}
\begin{align}
&\leqslant 2\beta_T  \sum_{i=1}^{M} \sum_{t=t_{0}}^{t_{0} + n} \min \{\| \phi_{x_{\ti}, c_{t}} \|_{\bar{V}_{\ti}^{-1}}, 1 \}\nonumber \\
&=  2\beta_T  \sum_{i=1}^{M} \sqrt { ( \sum_{t=t_{0}}^{t_{0} + n} \min \{\| \phi_{x_{\ti}, c_{t}} \|_{\bar{V}_{\ti}^{-1}}, 1 \} )^{2}} \nonumber \\
&\leqslant \beta_T  \sum_{i=1}^{M} \sqrt {n \sum_{t=t_{0}}^{t_{0} + n} \min \{\| \phi_{x_{\ti}, c_{t}} \|_{\bar{V}_{\ti}^{-1}}^{2}, 1 \} } \nonumber
\end{align}
\begin{align}
&\leqslant 2 \beta_T  \sum_{i=1}^{M} \sqrt {n \sum_{t=t_{0}}^{t_{0} + n} \min \{\| \phi_{x_{\ti}, c_{t}} \|_{\bar{V}_{t-1, i}^{-1}}^{2}, 1 \} } \nonumber \\
& \leqslant 2 \beta_T  \sum_{i=1}^{M} \sqrt{n \log \frac{\mathrm{det} (V_{t_{0}+n, i})}{\mathrm{det} (V_{last})}} \leqslant 2\beta_T M \sqrt{B}\label{eq:bad-bound-2}
\end{align}
The last three steps in the inequalities above follow from $(\sum_{i=1}^{n} z_{i})^{2} \leqslant n \sum_{i=1}^{n} z_{i}^{2}$, $ \bar{V}_{\ti}^{-1}  \leqslant \bar{V}_{i, t-1}^{-1} $,  Proposition~\ref{LT2}, and the fact that all agents in the bad epochs satisfy $ n \log  ({\det (V_{t_{0}+n, i})} /{\det (V_{last})}) < B $. 
We know  the number of bad epochs is at most $R= \left\lceil d\log(MT)^2 \right\rceil$.
Thus from Lemma~\ref{lem:Dj-bound} with $ n = MT_\bad$ and Lemma~\ref{lem:Sj} ($\lambda=1$) and after applying a union bound, we can rewrite  \eqref{eq:bad-bound-2} with $\beta_\ti = \beta_\ti(\sigma, \delta/3)$ as
\begin{align}
\Rbad &\leqslant \Big (2\beta_T  M  \sqrt{B} \Big) R + 4 (\beta_T + 1) \sqrt{2 M T_\bad \log \frac{3}{\delta}} \nonumber \\
&\leqslant 4 d \log(MT) \Big (\beta_T  M \sqrt{B}\Big) + 4 (\beta_T + 1) \sqrt{2 M T_\bad \log \frac{3}{\delta}}. \nonumber 
\end{align}
%
\end{proof}

\begin{proof}[Proof of Theorem~\ref{thm:main-2}]
Using Theorem~\ref{thm:good-2} and Theorem~\ref{thm:bad-2} and following  similar steps as in the proof of Theorem~\ref{thm:main-2} along with $ B = (\frac{T \log MT}{d M}) $ and  $ T > M $ we get
\begin{align*}
&\R(T) = \R_\good(T)+\R_\bad(T)\\
& \leqslant 4\beta_T \sqrt { MT_\good d  \log(MT)}  +4 (\beta_T+1)\sqrt{2MT_\good \log \frac{3}{\delta}} \\
& + 4 d \log(MT) \Big (\beta_T  M \sqrt{B}\Big)+ 4 (\beta_T + 1) \sqrt{2 M T_\bad \log \frac{3}{\delta}} \\
& = 4 \beta_T \sqrt{MT d \log (MT)}+ 4 d \log(MT) \Big (\beta_T  M \sqrt{B}\Big)\\
& + 4 (\beta_T + 1)  \sqrt{2 M T \log \frac{3}{\delta}}.
\end{align*}
Let us choose $ B = (\frac{T \log MT}{d M}) $, then 
\begin{align*}
 \R(T) \leqslant & 4 \beta_T \sqrt{MT d \log (MT)}(1+ \log (MT)) \\
 &+ 4 (\beta_T +1) \sqrt{2 M T \log \frac{3}{\delta}}\\
\hspace*{-15 mm}&= O(d\sqrt{MT}\log(MT))+O(d\sqrt{MT}\log^2(MT)).
\end{align*}
The last step uses $\beta_T = O \Big( \sqrt{ d \log(\frac{ T }{ \delta})} \Big) $ with $\delta = 1/M^2T$.
 Since $ T > M $, we have
\begin{align*}
\R(T) &= O (d \sqrt{M T} \log (T) ) + O (d \sqrt{M T} \log ^{2}(T) )\\
& = O (d \sqrt{M T} \log ^{2}(T) ). 
\end{align*}
\begin{figure*}[ht!]
 \subcaptionbox{Synthetic data: exact and delayed with $M=3$\label{fig:1}}{\includegraphics[width=0.33\textwidth]{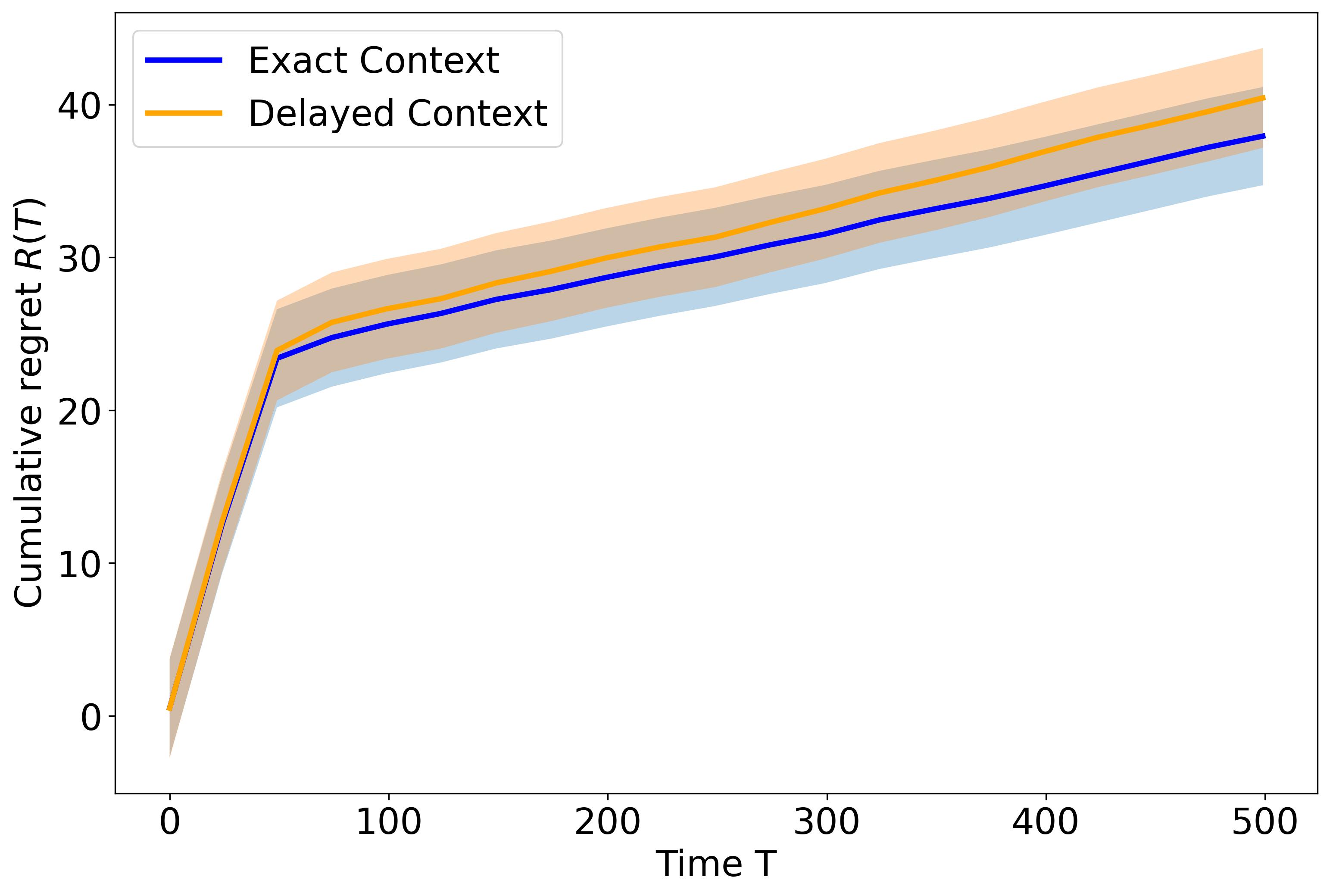}}
 \subcaptionbox{Synthetic data: $M=1, 3, 5$\label{fig:2}}{\includegraphics[width=0.33\textwidth]{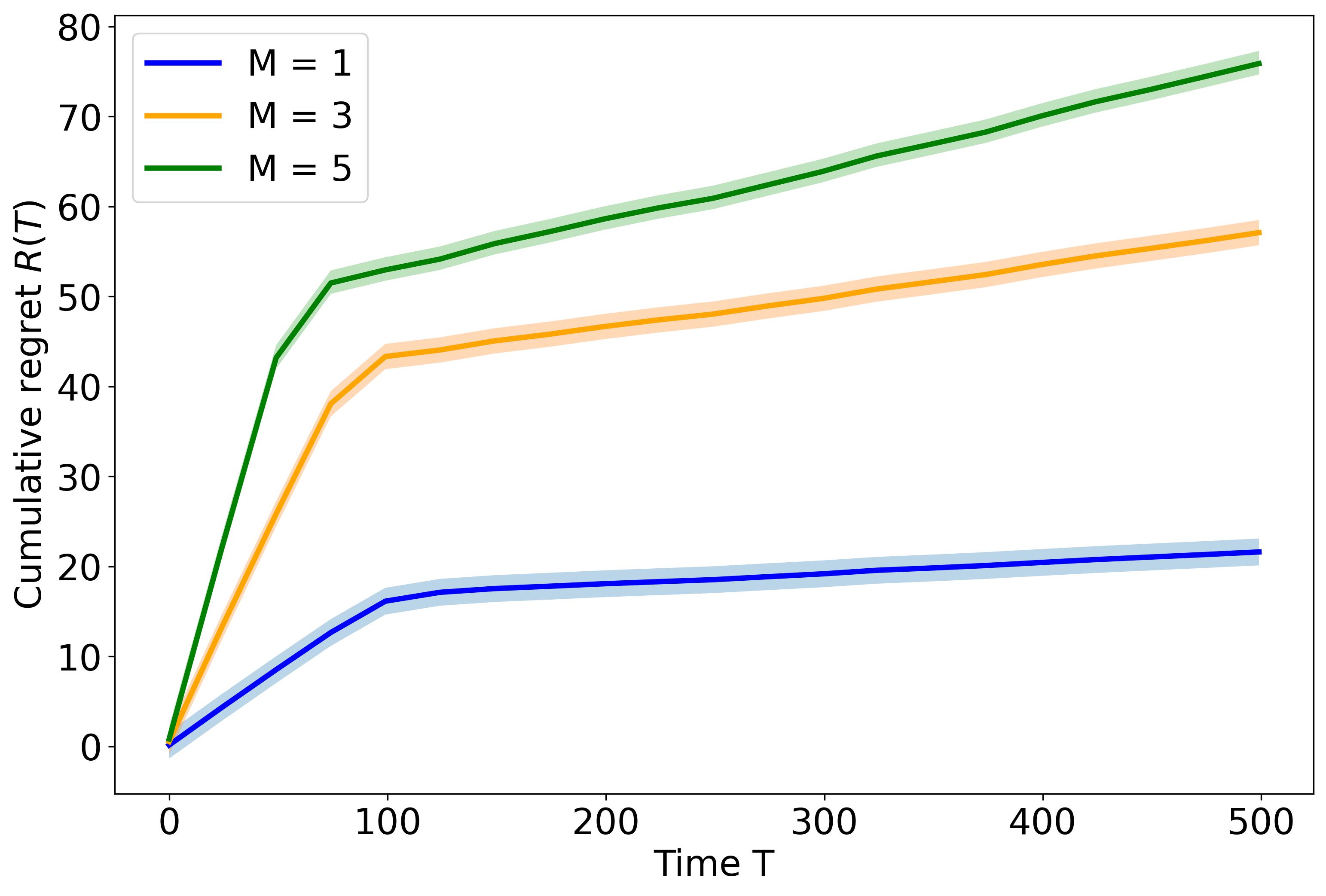}}
 \subcaptionbox{Movielens data: exact and delayed with $M=3$\label{fig:3}}{\includegraphics[width=0.33\textwidth]{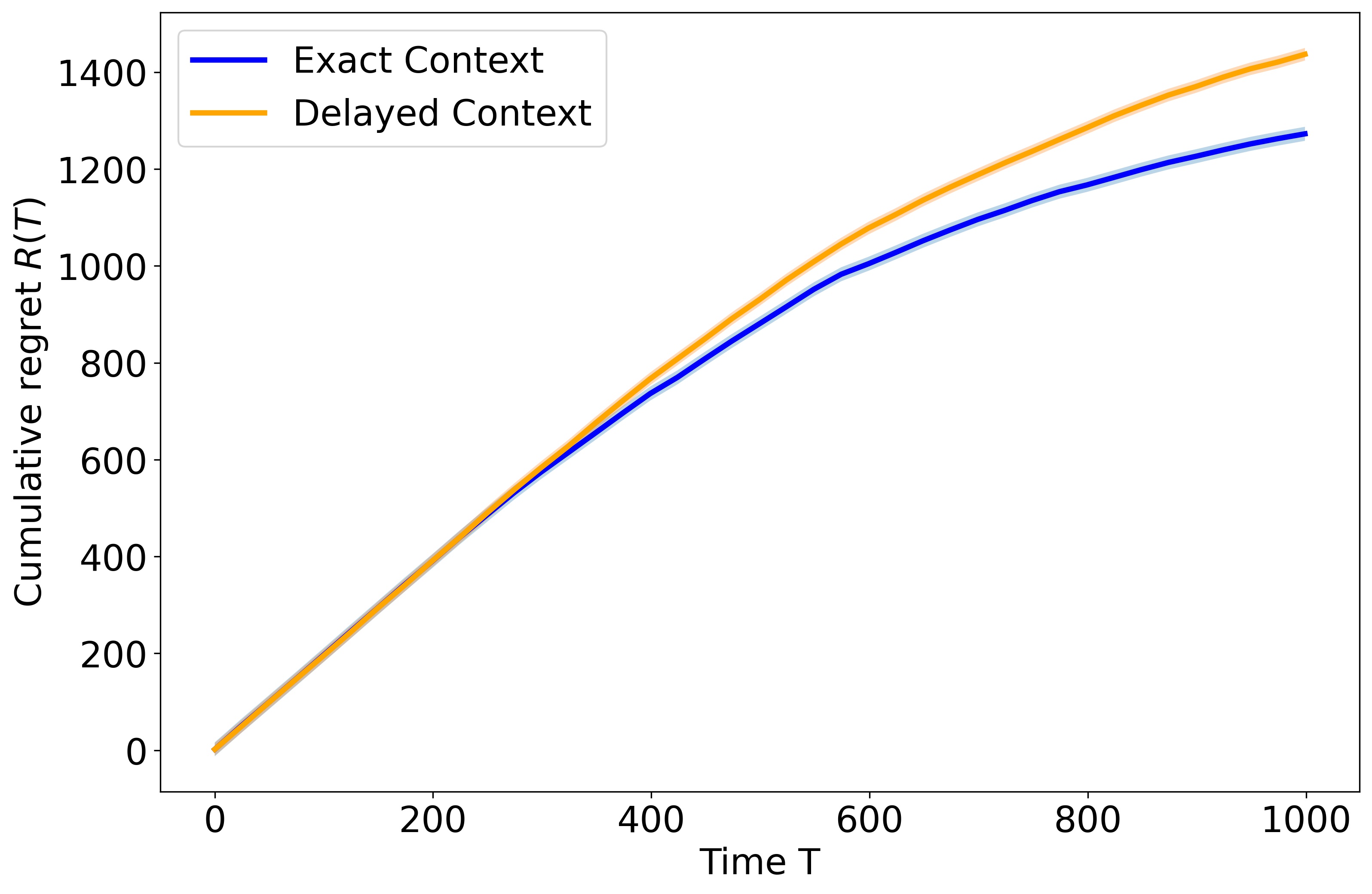}}
\caption{\small Cumulative regret $\R(T)$ versus time $T$ for synthetic data and the movielens data.  We compare the performance of our algorithm (delayed) with the DisLinUCB algorithm in \cite{wang2019distributed} (the variant in which  the actual context is observable, i.e., exact) in Figure~\ref{fig:1} and Figure~\ref{fig:3}  for the synthetic data and the movielens data, respectively. As expected  {\em exact}   outperforms the other two variants, and {\em observed} (the variant which utilizes the context information for estimation but not for choosing the actions) outperforms  {\em delayed}. Figure~\ref{fig:2} compares the cumulative regret for different number of agents for synthetic data, we chose $M=1, 3, 5$.  For $M=1$ our algorithm simplifies to the algorithm in \cite{kirschner2019stochastic}.}\label{fig:Syn}
\end{figure*}

{\bf Communication:} The communication cost for our algorithm follows the approach in \cite{wang2019distributed}. We present it here for completeness. Recall $ R = O(d \log (MT)) $. Let $ \alpha = \sqrt { \frac{BT}{R} } $. This implies that there could be at most $ \lceil T / \alpha\rceil = \lceil \sqrt { \frac{TR}{B} } \rceil$ epochs that contains more than $ \alpha $ time steps. Consider epochs $p$ and $p+1$.  If the $ p$-th epoch contains less than $ \alpha $ time steps, we get
\begin{align*}
\log (\frac{ \det (V_{p+1})}{\det (V_{p})}) (t_{p+1} - t_{p}) & > B\\
\log (\frac{\det (V_{p+1})}{\det (V_{p})}) > \frac{B}{t_{p+1} - t_{p}} &> \frac{B}{\alpha}.
\end{align*}
Further, 
\[
\sum_{p=0}^{P-1} \log ( \frac {\det(V_{p+1})} {\det(V_{p})} ) = \log \frac{\det(V_{P})}{\det(V_{0})} \leqslant R
\]
As a result, there could be at most $ \left\lceil\frac{R}{B / \alpha}\right\rceil = \left\lceil\frac{R \alpha}{B}\right\rceil = \lceil \sqrt { \frac{TR}{B} } \rceil $ epochs with less than $ \alpha $ time steps. Therefore, the total number of epochs is at most
\[
\left\lceil\frac{T}{\alpha}\right\rceil + \left\lceil\frac{R \alpha}{B}\right\rceil = \left\lceil \sqrt { \frac{TR}{B} } \right\rceil + \left\lceil \sqrt { \frac{TR}{B} } \right\rceil = O (\sqrt{\frac{TR}{B}} ) 
\]
By selecting $ B = (\frac{T \log MT}{d M}) $ and $ R = O(d \log M T)) $, the right-hand-side is $ O(M^{0.5} d) $. The agents communicate only at the end of each epoch, when each agent sends $ O(d^{2}) $ numbers to the server, and then receives $ O(d^{2}) $ numbers from the server. Therefore, in each epoch, communication cost is $ O(M d^{2}) $. Hence, total communication cost is $ O(M^{1.5} d^{3}) $.
\end{proof}

\section{Simulation Results}\label{sec:sim}
We evaluated the performance of our  algorithm via numerical experiments using synthetic data and  a real-world Movielens dataset.  All the experiments were done in Python.  We compared two  variants of the problem, based on the observation of the agents. (i)~{\em Exact} in which the agents know the context, i.e.,  the agents observe the actual context before selecting the action (which is the distributed bandit setting considered in \cite{wang2019distributed}).  (ii)~{\em Delayed}  in which the contexts are hidden at the time of selection actions, however, observed afterward (studied in \cite{kirschner2019stochastic}  for single-agent setting). 

\noindent{\bf Synthetic data:} 
In this dataset we generated the data by setting $K = 10$, $d = 3$. We  ensured that the suboptimal reward gap and norm of $\| \phi_{x, c} \|_2$ are in  $[0.2, 0.4]$ and $[0.5, 1]$, respectively. Each point in the plot was averaged over $20$ independent trials. We present the plots showing variations of the cumulative regret with respect to the execution time for synthetic data for different variants in Figure~\ref{fig:1} and for different numbers of agents in Figure~\ref{fig:2}. Our experimental results given in Figure~\ref{fig:1} shows that the exact setting outperforms the delayed setting, as expected since the agents observe the actual contexts. We varied the number of agents as $M = 1,3,5$ and compared the results as shown in Figure~\ref{fig:2}.  We ran for $500$ time-period and $20$ independent trails. 

\noindent{\bf Movielens data:}  We used MovieLens data to evaluate the performance of our algorithm. For the rating matrix $R = [r_{x, c} ]\in \mathbb{R}^{943 \times 1682}$ of the data, we first obtained a non-negative matrix factorization $R = WH$, where $W \in \mathbb{R}^{943 \times 6}$, $H \in \mathbb{R}^{6 \times 1682}$ \cite{bogunovic2021stochastic}. Each row of $W$, $\{W_j^\top\}_{j\in[943]}$, represents a context and each column of $H$, $\{H_k\}_{k\in[1682]}$, represents an action. The feature vector for a given context $W_j\in \mathbb{R}^{6}$ and action $H_k \in \mathbb{R}^{6}$ is given by the  diagonal of the matrix  $W_j H_k^\top$. Hence the feature vector is of dimension $6$ and $\thetas = [1, 1, 1, 1, 1, 1]$. We chose $100$ actions randomly from the action set. We present the plots showing the variation of the cumulative regret with respect to the execution time for the MovieLens data for different settings in Figure~\ref{fig:3}. The reward $r(x_t, c_t)$ is bounded above by $1$, and the observation noise $\eta_i$ is set as Gaussian with zero mean and standard deviation $10^{-3}$. In this experiment, as expected, the  exact setting outperforms the delayed setting. We ran for a time-period of $1000$ and the plots are shown  in Figure~\ref{fig:3}. 

\section{Conclusion}\label{sec:con}
In this work, we  studied distributed stochastic multi-arm contextual bandit problem when the contexts are observed after a delay and only a distribution on the contexts is available at the time of decision. In our distributed setting,  $M$ agents face the same MAB problem and work collaboratively to choose optimal actions  to minimize the total cumulative regret.  We leveraged the feature vector transformation in \cite{kirschner2019stochastic} and proposed a   UCB-based algorithm and proved  regret and communications bounds as $O(d\sqrt{MT}\log^2 T)$  and $ O(M^{1.5} d^{3}) $, respectively, for linearly parametrized reward functions.  To validate the performance of our approach we performed numerical simulations on synthetic data and on Movielens data set. As a part of our future work, we plan investigate the fully-decentralized setting of the problem in which the agents are allowed to communicate only with their neighboring agents defined by a communication network.

\bibliographystyle{myIEEEtran}
\bibliography{../../bib/Bandits}

\begin{thebibliography}{10}
\providecommand{\url}[1]{#1}
\csname url@rmstyle\endcsname
\providecommand{\newblock}{\relax}
\providecommand{\bibinfo}[2]{#2}
\providecommand\BIBentrySTDinterwordspacing{\spaceskip=0pt\relax}
\providecommand\BIBentryALTinterwordstretchfactor{4}
\providecommand\BIBentryALTinterwordspacing{\spaceskip=\fontdimen2\font plus
\BIBentryALTinterwordstretchfactor\fontdimen3\font minus
  \fontdimen4\font\relax}
\providecommand\BIBforeignlanguage[2]{{%
\expandafter\ifx\csname l@#1\endcsname\relax
\typeout{** WARNING: IEEEtran.bst: No hyphenation pattern has been}%
\typeout{** loaded for the language `#1'. Using the pattern for}%
\typeout{** the default language instead.}%
\else
\language=\csname l@#1\endcsname
\fi
#2}}

\bibitem{cheung2013autonomous}
M.~Y. Cheung, J.~Leighton, and F.~S. Hover, ``Autonomous mobile acoustic relay
  positioning as a multi-armed bandit with switching costs,'' in \emph{IEEE/RSJ
  International Conference on Intelligent Robots and Systems}, 2013, pp.
  3368--3373.

\bibitem{srivastava2014surveillance}
V.~Srivastava, P.~Reverdy, and N.~E. Leonard, ``Surveillance in an abruptly
  changing world via multiarmed bandits,'' in \emph{IEEE Conference on Decision
  and Control (CDC)}, 2014, pp. 692--697.

\bibitem{anandkumar2011distributed}
A.~Anandkumar, N.~Michael, A.~K. Tang, and A.~Swami, ``Distributed algorithms
  for learning and cognitive medium access with logarithmic regret,''
  \emph{IEEE Journal on Selected Areas in Communications}, vol.~29, no.~4, pp.
  731--745, 2011.

\bibitem{srivastava2013optimal}
V.~Srivastava, P.~Reverdy, and N.~E. Leonard, ``On optimal foraging and
  multi-armed bandits,'' in \emph{Annual Allerton Conference on Communication,
  Control, and Computing (Allerton)}, 2013, pp. 494--499.

\bibitem{bubeck2012regret}
S.~Bubeck and N.~Cesa-Bianchi, ``Regret analysis of stochastic and
  nonstochastic multi-armed bandit problems,'' \emph{arXiv preprint
  arXiv:1204.5721}, 2012.

\bibitem{lattimore2020bandit}
T.~Lattimore and C.~Szepesv{\'a}ri, \emph{Bandit algorithms}.\hskip 1em plus
  0.5em minus 0.4em\relax Cambridge University Press, 2020.

\bibitem{huang2021federated}
R.~Huang, W.~Wu, J.~Yang, and C.~Shen, ``Federated linear contextual bandits,''
  \emph{Advances in Neural Information Processing Systems}, vol.~34, 2021.

\bibitem{wang2019distributed}
Y.~Wang, J.~Hu, X.~Chen, and L.~Wang, ``Distributed bandit learning:
  Near-optimal regret with efficient communication,'' \emph{arXiv preprint
  arXiv:1904.06309}, 2019.

\bibitem{landgren2021distributed}
P.~Landgren, V.~Srivastava, and N.~E. Leonard, ``Distributed cooperative
  decision making in multi-agent multi-armed bandits,'' \emph{Automatica}, vol.
  125, p. 109445, 2021.

\bibitem{kirschner2019stochastic}
J.~Kirschner and A.~Krause, ``Stochastic bandits with context distributions,''
  \emph{Advances in Neural Information Processing Systems}, vol.~32, pp.
  14\,113--14\,122, 2019.

\bibitem{abbasi2011improved}
Y.~Abbasi-Yadkori, D.~P{\'a}l, and C.~Szepesv{\'a}ri, ``Improved algorithms for
  linear stochastic bandits,'' \emph{Advances in Neural Information Processing
  Systems}, vol.~24, pp. 2312--2320, 2011.

\bibitem{kazerouni2016conservative}
A.~Kazerouni, M.~Ghavamzadeh, Y.~Abbasi-Yadkori, and B.~Van~Roy, ``Conservative
  contextual linear bandits,'' \emph{Advances in Neural Information Processing
  Systems}, 2017.

\bibitem{lamprier2018profile}
S.~Lamprier, T.~Gisselbrecht, and P.~Gallinari, ``Profile-based bandit with
  unknown profiles,'' \emph{The Journal of Machine Learning Research}, vol.~19,
  no.~1, pp. 2060--2099, 2018.

\bibitem{yun2017contextual}
S.-Y. Yun, J.~H. Nam, S.~Mo, and J.~Shin, ``Contextual multi-armed bandits
  under feature uncertainty,'' \emph{arXiv preprint arXiv:1703.01347}, 2017.

\bibitem{lin2022stochastic}
J.~Lin, X.~Y. Lee, T.~Jubery, S.~Moothedath, S.~Sarkar, and
  B.~Ganapathysubramanian, ``Stochastic conservative contextual linear
  bandits,'' \emph{IEEE Conference on Decision and Control (CDC)}, 2022.

\bibitem{yi2020distributed}
X.~Yi, X.~Li, T.~Yang, L.~Xie, T.~Chai, and K.~H. Johansson, ``Distributed
  bandit online convex optimization with time-varying coupled inequality
  constraints,'' \emph{IEEE Transactions on Automatic Control}, vol.~66,
  no.~10, pp. 4620--4635, 2020.

\bibitem{bistritz2018distributed}
I.~Bistritz and A.~Leshem, ``Distributed multi-player bandits-a game of thrones
  approach,'' \emph{Advances in Neural Information Processing Systems},
  vol.~31, 2018.

\bibitem{kalathil2014decentralized}
D.~Kalathil, N.~Nayyar, and R.~Jain, ``Decentralized learning for multiplayer
  multiarmed bandits,'' \emph{IEEE Transactions on Information Theory},
  vol.~60, no.~4, pp. 2331--2345, 2014.

\bibitem{rosenski2016multi}
J.~Rosenski, O.~Shamir, and L.~Szlak, ``Multi-player bandits--a musical chairs
  approach,'' in \emph{International Conference on Machine Learning}.\hskip 1em
  plus 0.5em minus 0.4em\relax PMLR, 2016, pp. 155--163.

\bibitem{nayyar2016regret}
N.~Nayyar, D.~Kalathil, and R.~Jain, ``On regret-optimal learning in
  decentralized multiplayer multiarmed bandits,'' \emph{IEEE Transactions on
  Control of Network Systems (IEEE-TCNS)}, vol.~5, no.~1, pp. 597--606, 2016.

\bibitem{bogunovic2021stochastic}
I.~Bogunovic, A.~Losalka, A.~Krause, and J.~Scarlett, ``Stochastic linear
  bandits robust to adversarial attacks,'' in \emph{International Conference on
  Artificial Intelligence and Statistics}.\hskip 1em plus 0.5em minus
  0.4em\relax PMLR, 2021, pp. 991--999.

\end{thebibliography}

\end{document}